\documentclass[letterpaper, 11 pt]{article} 
\usepackage{amsmath,amsfonts,amssymb,color}
\usepackage{mathtools}
\usepackage{dsfont}
\usepackage[dvipsnames]{xcolor}
\definecolor{mygreen}{rgb}{0.0, 0.5, 0.0}
\definecolor{winered}{rgb}{0.8,0,0}
\definecolor{myblue}{rgb}{0,0,0.8}
\usepackage{tikz}
\usepackage{amsthm}
\usepackage{empheq}

\newtheorem{theorem}{Theorem}
\newtheorem{lemma}{Lemma}
\newtheorem{proposition}{Proposition}

\newtheorem{assumption}{Assumption}

\newcommand{\mc}{\mathcal}
\newcommand{\mb}{\mathbf}

\DeclareMathOperator*{\argmin}{\arg\!\min}

\allowdisplaybreaks[4]
\newcommand{\E}{\mathbb{E}}

\newcommand{\cA}{\mathcal{A}}

\newcommand{\cF}{\mathcal{F}}
\newcommand{\cG}{\mathcal{G}}

\newcommand{\cM}{\mathcal{M}}

\newcommand{\cO}{\mathcal{O}}
\newcommand{\cP}{\mathcal{P}}

\newcommand{\cS}{\mathcal{S}}

\newcommand{\cV}{\mathcal{V}}
\newcommand{\cW}{\mathcal{W}}

\newcommand{\iprod}[2]{\left\langle #1,\, #2\right\rangle}
\newcommand{\bracket}[1]{\left[ #1\right]}
\newcommand{\norm}[1]{\left\| #1\right\|}
\newcommand{\bigo}[1]{\cO\left( #1\right)}
\usepackage{algorithm}
\usepackage{algpseudocode}
\usepackage{epsfig}
\usepackage{array}
\usepackage{multirow}
\usepackage{epstopdf}
\usepackage{tikz}
\usepackage{relsize}
\usetikzlibrary{shapes,arrows}
\usepackage[hidelinks]{hyperref}
\hypersetup{
    colorlinks=true,
    linkcolor={winered},
    citecolor={mygreen}
}
\usepackage{cite}
\usepackage[margin=1in]{geometry}
\title{Towards Fast Rates for Federated and Multi-Task Reinforcement Learning}
\author{Feng Zhu, Robert W. Heath Jr., and Aritra Mitra
\thanks{F. Zhu and A. Mitra are with the Dept. of Electrical and Computer Engineering, North Carolina State University. Email: {\tt \{fzhu5, amitra2\}@ncsu.edu}. Robert W. Heath Jr. is with the Dept. of Electrical and Computer Engineering at the University of California, San Diego, USA. Email: {\tt rwheathjr@ucsd.edu}. This material is based upon work supported in part by the National Science Foundation under Grant No. NSF-CCF-2225555.}}
\date{}
\begin{document}
\maketitle
\thispagestyle{empty}
\pagestyle{empty}
\begin{abstract}
    We consider a setting involving $N$ agents, where each agent interacts with an environment modeled as a Markov Decision Process (MDP). The agents' MDPs differ in their reward functions, capturing heterogeneous objectives/tasks. The collective goal of the agents is to communicate intermittently via a central server to find a policy that maximizes the average of long-term cumulative rewards across environments. The limited existing work on this topic either only provide asymptotic rates, or generate biased policies, or fail to establish any benefits of collaboration. In response, we propose \texttt{Fast-FedPG} - a novel federated policy gradient algorithm with a carefully designed bias-correction mechanism. Under a gradient-domination condition, we prove that our algorithm guarantees (i) fast linear convergence with exact gradients, and (ii) sub-linear rates that enjoy a linear speedup w.r.t. the number of agents with noisy, truncated policy gradients. Notably, in each case, the convergence is to a globally optimal policy with no heterogeneity-induced bias. In the absence of gradient-domination, we establish convergence to a first-order stationary point at a rate that continues to benefit from collaboration. 
\end{abstract}
\section{Introduction}
Despite the many successes of reinforcement learning (RL) in various applications (e.g., games, robotics, autonomous navigation, etc.), a large part of existing RL theory only provides asymptotic rates. Recently however, there has been a surge of interest in characterizing the finite-time behavior of model-free RL algorithms. For contemporary RL applications with massive state and action spaces, such finite-time analysis has revealed the need for lots of data samples to achieve desirable performance. Given this premise, it is natural to wonder if data collected from diverse environments can alleviate the sample-complexity bottleneck. This has prompted the emergence of a new paradigm called federated reinforcement learning (FRL),  where agents interacting with potentially distinct environments collaborate with the hope of learning ``good" policies with fewer samples than if they acted alone~\cite{FRL}. Unfortunately, existing work in FRL either only provide empirical results~\cite{FRL}, or make the unrealistic assumption of identical agent environments~\cite{FRL_identical_linear, dal2023federated, liu2023distributed, lan2023improved, woo2023blessing, tian2024one}, or provide rates that exhibit a non-vanishing environmental-heterogeneity-induced bias term~\cite{jin2022federated, fedtd, zhang2024finite}. In particular, such an additive bias term negates potential statistical gains from collaboration. In this paper, \emph{we show for the first time that achieving collaborative speedups in FRL is possible even when data is collected from non-identical environments.}

\textbf{Our model.} We consider a sequential decision-making setting involving $N$ agents, where each agent's environment is modeled as a Markov Decision Process (MDP). The agents' MDPs share the same state and action spaces, have identical probability transition maps, but differ in their reward functions; the non-identical reward functions help capture different goals/tasks across environments. The agents collaborate via a central server to learn a policy that can perform well in all environments by maximizing an average of the agents' long-term cumulative rewards. In this sense, our work is also related to multi-task RL, where data from different tasks is used to improve the performance on any given task~\cite{sodhani}. As in the standard FL setting~\cite{mcmahan}, to achieve communication-efficiency, the agents are allowed to communicate only once in every $H$ iterations. Furthermore, to respect privacy, agents are not allowed to reveal their raw data (i.e., states, actions, and rewards). With this setup, we formulate a heterogeneous federated policy optimization problem. Recently, the authors in \cite{xie2023fedkl, jin2022federated, zeng2021decentralized} have explored the effects heterogeneity in the context of federated/decentralized policy gradient (PG) methods. While~\cite{xie2023fedkl} only provides asymptotic rates, \cite{jin2022federated} and~\cite{zeng2021decentralized} fail to establish any provable benefits of collaboration. In this context, our \textbf{contributions} are as follows.

 $\bullet$ \textbf{New algorithm}. We propose a novel federated PG algorithm called \texttt{Fast-FedPG} that, unlike standard ``model-averaging" algorithms~\cite{FRL_identical_linear, fedtd, jin2022federated, zeng2021decentralized, zhang2024finite} typically used in FRL, relies on a carefully constructed de-biasing/drift-mitigation mechanism using memory. Such a mechanism has not been explored earlier in FRL. 

$\bullet$  \textbf{Key structural result}. To establish fast rates, we prove a simple, yet key structural result (Proposition~\ref{prop:key_result}) that relates the gradient of our objective function to the policy gradient of an ``average MDP" constructed from the agents' MDPs. 

$\bullet$  \textbf{Fast rates and linear speedup}. Under a gradient-domination condition used to prove fast rates for centralized PG methods~\cite{meiglobal, yuan2022}, we prove that \texttt{Fast-FedPG} guarantees linear convergence to a globally optimal policy with exact gradients. With noisy, truncated policy gradients, we prove a rate of $\tilde{O}(1/(NHT))$ after $T$ communication rounds, with $H$ local PG steps within each round; see Theorem~\ref{thm:fastrate}. Notably, our rates feature no heterogeneity-induced bias, and exhibit a clear $N$-fold speedup w.r.t. the number of agents, \emph{thereby providing the first collaborative speedup result in FRL despite heterogeneity.} Finally, in Theorem~\ref{thm:statpt}, we show that in the absence of gradient-domination, \texttt{Fast-FedPG} guarantees convergence to a first-order stationary point at a rate of $\tilde{O}(1/\sqrt{NHT})$, i.e., with a $\sqrt{N}$-fold speedup.

{
\textbf{Discussion of Concurrent Work.} We would like to point to a couple of pieces of closely related concurrent work~\cite{bai2024finite, wang2024momentum} that appeared on arXiv after the submission of our paper to the Decision and Control Conference (CDC), 2024. Each of these papers studies the same problem as us. The rates derived in~\cite{bai2024finite} are slower than in our paper; moreover, no linear speedup result is established in~\cite{bai2024finite}. The authors in~\cite{wang2024momentum} propose an algorithm where each agent employs momentum locally. Like us, they establish linear speedup results with no heterogeneity-induced bias term. Our algorithm and results differ from those in~\cite{wang2024momentum} in the following ways. First, one key message conveyed by our work is that \emph{momentum is neither needed to achieve a linear speedup nor to eliminate the effect of a heterogeneity-induced bias term.} Second, the results in~\cite{wang2024momentum} concern convergence to first-order stationary points. In contrast, by leveraging the key structural result in Proposition~\ref{prop:key_result}, we show how our approach can also guarantee fast convergence to an optimal policy parameter. Third, in terms of technical assumptions, the results in~\cite{wang2024momentum} seem to require uniform boundedness of the first and second derivatives of the log-density of the policy function (Assumption 6.1 in~\cite{wang2024momentum}). We do not need such an assumption in our analysis. Finally, we note that by using momentum, the authors in~\cite{wang2024momentum} are able to prove faster rates than us when it comes to convergence to a first-order stationary point. In particular, the rate of convergence to a first-order stationary point is $\tilde{O}(1/(NHT)^{2/3})$ in~\cite{wang2024momentum}; in contrast, our corresponding rate in Theorem~\ref{thm:statpt} is slower: it is $\tilde{O}(1/\sqrt{NHT})$.} 

\section{Problem Formulation}
\label{sec:prob_form}
We start by describing our RL setting, and then introduce the policy gradient method to formulate our problem of interest. 

\noindent \textbf{RL setting.} Our setting involves $N$ agents, where each agent $i$ interacts with an environment characterized by an MDP $\mc{M}_i = (\cS, \cA, R_i, \mc{P}, \gamma).$ Here, $\cS$ is a finite state space, $\mc{A}$ is a finite action space, $R_i:\cS\times\cA\rightarrow[0, 1]$ is a bounded reward function \emph{specific to agent} $i$ where $R_i(s,a)$ represents the immediate expected reward for taking action $a$ in state $s$, $\mc{P}$ is a Markovian transition model where $\mc{P}(s'| s,a)$ represents the probability of transitioning from state $s$ to $s'$ under action $a$, and $\gamma \in [0,1)$ is a discount factor. Therefore, agents share the same state and action spaces, are governed by the same probability transition maps, but have potentially different goals/objectives as captured by their unique reward functions.\footnote{{The assumption that the agents share the same state transition kernels is only needed to prove fast linear convergence to an optimal policy parameter (Theorem~\ref{thm:fastrate}). If one only cares about sub-linear convergence to a stationary point as in Theorem~\ref{thm:statpt}, then this assumption is no longer needed for our analysis to go through.}} The distinction in the reward functions captures \emph{heterogeneity} across the agents' environments.  

The behavior of an agent is captured by a stochastic policy $\pi: \mc{S} \mapsto \Delta(\mc{A})$, where $\Delta(\mc{A})$ is the space of probability distributions over $\mc{A}.$ The dynamics of an agent-MDP interaction process unveils as follows. Starting from some initial state $s_{i}^{(0)}$, suppose an agent $i$ interacts with its MDP $\mc{M}_i$ by playing a particular policy $\pi$. In particular, at each time-step $t=0, 1, 2, \ldots$, the agent plays $a_{i}^{(t)} \sim \pi(\cdot|s_{i}^{(t)})$, observes an immediate reward $r_{i}^{(t)}=R_i(s_{i}^{(t)}, a_{i}^{(t)})$, and transitions to a new state $s_{i}^{(t+1)} \sim \mc{P}(\cdot| s_{i}^{(t)}, a_{i}^{(t)})$. This repeated interaction process generates a trajectory $\tau_i = \{(s_{i}^{(0)}, a_{i}^{(0)}, r_{i}^{(0)}), (s_{i}^{(1)}, a_{i}^{(1)}, r_{i}^{(1)}), \cdots\}.$ In the single-agent RL setting, the typical goal of the agent $i$ would be to find a policy $\pi$ that maximizes a $\gamma$-discounted infinite-horizon expected cumulative reward, given by
\begin{equation}
    J_i(\pi) \triangleq \E\bracket{\sum_{t=0}^\infty \gamma^t r_{i}^{(t)}\Big|s_i^{(0)}\sim\rho, \pi},
\label{eqn:value_func}
\end{equation}
where $\rho$ is an initial state distribution, and the expectation is taken w.r.t. the randomness in the initial state, the randomness induced by the stochastic policy $\pi$, and the randomness due to the state transitions prescribed by $\mc{P}.$ For simplicity, we will assume throughout that all agents start from the same initial state distribution $\rho$. When the dynamics of the MDP are known, an optimal policy can be found using dynamic programming~\cite{puterman}. The learning aspect in our problem, however, stems from the fact that the {reward functions $\{R_i\}_{i\in [N]}$ and} state transition maps $\mc{P}$ are \emph{unknown} to the agent. Given the fact that PG methods are easy to implement, we now describe the policy optimization approach for finding optimal policies that belong to a parameterized class. 

\noindent \textbf{Policy Gradient (PG) methods.} Consider a class of parametric policies $\{\pi_{\theta}: \theta \in \mathbb{R}^d\}$, where $\pi_{\theta}$ is assumed to be differentiable w.r.t. $\theta$. A common example of such a class is the \emph{softmax policy}: 
\begin{align}
    \pi_\theta(a | s) = \frac{\exp(\theta_{s,a})}{\sum_{a'\in\cA}\exp(\theta_{s,a'})},
\end{align}
where the parameter space is $\mathbb{R}^{|\mc{S}| |\mc{A}|}.$ For other common parametric classes (e.g., log-linear, neural softmax, etc.,), we refer the reader to~\cite{agarwalPG}. Given a parameterized policy $\pi_{\theta}$, let $J_i(\theta) \triangleq J_i(\pi_{\theta})$ be agent $i$'s local value-function associated with the parameter $\theta$; here, $J_i(\cdot)$ is as defined in Eq.~\eqref{eqn:value_func}. Policy gradient  methods operate by incrementally updating the parameter $\theta$ by performing gradient ascent on the value function. 

\noindent \textbf{Goal.} Informally, we seek to find a policy $\pi_{\theta}$ that performs ``well" on the set of environments $\{\mc{M}_i\}_{i\in[N]}$. This formulation is inspired by the federated supervised learning setting where agents with access to data from different distributions collaborate to find models with superior statistical performance relative to models trained with just individual agent-data. To formally set up our problem using the language of optimization, for each $i\in [N]$ and $(s,a) \in \mc{S} \times \mc{A}$, we reset $R_i(s,a) \leftarrow 1-R_i(s,a)$, and interpret $R_i(\cdot, \cdot)$ as a \emph{regret} function instead of a reward function. The collective goal of the agents then is to find a policy parameter $\theta^* \in \argmin_{\theta \in \mathbb{R}^d} J(\theta)$, where $J(\theta)$ is a global value-function defined as
\begin{equation}
\label{eqn:problem}
     J(\theta) \triangleq \frac{1}{N}\sum_{i=1}^N J_i(\theta). 
\end{equation}
To achieve the above objective within a federated framework, the agents can exchange information via a central server that coordinates the learning process. As in the FL setting, however, the agents need to adhere to stringent communication and privacy constraints, i.e., they are only allowed to communicate \emph{intermittently}, and are required to keep their raw data (i.e., states, actions, and rewards) private. We now discuss the key challenges in the problem posed above.  

$\bullet$ \textbf{Effect of reward-heterogeneity.} Since the agents have different reward functions, a locally optimal policy parameter $\theta^*_i \in \argmin_{\theta \in \mathbb{R}^d} J_i(\theta)$ for agent $i$ may not coincide with the globally optimal parameter $\theta^*$. Therefore, in the intermittent periods where the agents act locally to respect communication constraints, they will tend to drift towards their own locally optimal parameters. In this context, while \emph{drift-mitigation} techniques have been explored for federated supervised learning, their effectiveness remains unclear in  our RL setting. 

$\bullet$ \textbf{Effect of non-convexity.} As shown in~\cite{agarwalPG}, the value-function $J_i(\theta)$ is non-convex w.r.t. $\theta$ for even direct and softmax parameterizations. This precludes the use of standard tools from convex optimization theory for our purpose, making it highly non-trivial, in particular, to guarantee convergence to the globally optimal parameter $\theta^*$ in our heterogeneous federated RL setting. 

$\bullet$ \textbf{Effect of noise and truncation.} Policy gradients are typically \emph{noisy} and \emph{biased}. To see why, let us fix an agent $i \in [N]$, and note that based on the celebrated Policy Gradient Theorem~\cite{suttonpolicy}, the ideal exact gradient $\nabla J_i(\theta)$ is given by 

\begin{equation}
    \nabla J_i(\theta) = \E_{\tau_i}\bracket{\sum_{t=0}^\infty \gamma^t r_{i}^{(t)}\sum_{k=0}^\infty \nabla_{\theta}\log\pi_\theta(a_{i}^{(k)}\big|s_{i}^{(k)})},
\label{exact_gradient}
\end{equation}
where the expectation is w.r.t. the random trajectory $\tau_i$. There are two key issues that impede computing the exact gradient. First, computing the expectation in Eq.~\eqref{exact_gradient} would require averaging over all possible trajectories; this is infeasible. Second, during implementation, agents do not have the luxury of rolling out/simulating a trajectory of infinite length. Therefore, complying with practice, each agent $i$ computes an empirical estimate of $\nabla J_i(\theta)$ by sampling a truncated trajectory of length $K \in \mathbb{N}$: this is done by playing policy $\pi_{\theta}$ on MDP $\mc{M}_i$ over a finite roll-out horizon $K$. This leads to the following \emph{noisy} and \emph{biased} estimate of $\nabla J_i(\theta)$ that gets implemented in practice:
\begin{align}
    \hat{\nabla}_K J_i(\theta) = \sum_{t=0}^{K-1} \gamma^t r_{i}^{(t)}\sum_{k=0}^{K-1} \nabla_{\theta}  \log\pi_\theta(a_{i}^{(k)}|s_{i}^{(k)}),
\label{eqn:trunc_gradient}
\end{align}
where the noise arises due to sampling, and the bias due to truncation. For use later in the paper, let us also define the truncated gradient $\nabla_K J_i(\theta)$ as the expectation of the noisy truncated gradient, i.e., $\nabla_K J_i(\theta) \triangleq \E\bracket{\hat{\nabla}_K J_i(\theta)}.$
\vspace{1mm} \\
\newpage 
\noindent \textbf{Desiderata.} Despite the complex interplay between infrequent communication, client-drift effects due to reward heterogeneity, non-convex optimization landscapes, and inexact, truncated gradients, we seek to develop a federated PG method that (i) leads to \emph{de-biased solutions}, i.e., guarantees convergence to $\theta^*$, as opposed to $\theta^*_i$ for any $i \in [N]$; and (ii) achieves \emph{near-optimal statistical rates} that clearly exhibit the benefit of collaboration among agents. In the next section, we will design such an algorithm. 

\section{Fast Federated Policy Gradient}
\begin{algorithm}[t]
\caption{\texttt{Fast-FedPG}} 
\label{algo:FedAPG}
\begin{algorithmic}[1]
\State \textbf{Input:} Local step-size $\eta$, Global step-size $\alpha_g$, Initial parameter $\bar\theta^{(0)} \in \mathbb{R}^d$, Initial global PG $\hat{\nabla}_KJ(\Bar{\theta}^{(0)})$. 
\For {$t=0,\ldots,T-1$} 
\For {$i=1,\ldots, N$} 
\State Agent $i$ initializes its local parameter $\theta_{i,0}^{(t)}=\bar\theta^{(t)}.$
\For {$\ell=0,\ldots, H-1$}
\State \hspace{-2mm} Agent $i$ samples a truncated trajectory by playing policy 
$\pi_{\theta_{i,\ell}^{(t)}}$ on its MDP $\mc{M}_i$ over a horizon of length $K$. It then computes $\hat\nabla_K J_i(\theta_{i,\ell}^{(t)})$ as per Eq.~\eqref{eqn:trunc_gradient}. 
\State \hspace{-2mm} Agent $i$ updates local parameter as per Eq.~\eqref{update_FAPG}. 
\EndFor
\State Agent $i$ transmits $\Delta_{i,H}^{(t)}=\theta_{i,H}^{(t)}-\bar{\theta}^{(t)}$ to server. 
\EndFor
\State Server broadcasts $\bar\theta^{(t+1)}$ computed as per Eq.~\eqref{eqn:serv_aggr}.
\For {$i=1,\ldots, N$}
\State Agent $i$ transmits $\hat{\nabla}_K J_i(\Bar{\theta}^{(t+1)})$ to server.
\EndFor
\State Server broadcasts global PG $\hat{\nabla}_K J(\Bar{\theta}^{(t+1)})$.
\EndFor
\end{algorithmic}
\end{algorithm}

In this section, we will develop our proposed algorithm called \texttt{Fast Federated Policy Gradient} (\texttt{Fast-FedPG}), formally described in Algorithm \ref{algo:FedAPG}. The primary components of our algorithm involve \emph{local policy gradient steps}, and a \emph{de-biasing/drift mitigation strategy}. We now proceed to elaborate on these ideas. 

\textbf{Local policy gradient steps.} The structure of \texttt{Fast-FedPG} mimics a typical FL algorithm: it operates in rounds $t=0, 1, \ldots, T-1,$ where within each round, every agent performs $H$ local policy optimization steps in parallel by interacting with its own environment. During these local steps, there is no communication with the server. Let us denote by $\theta_{i,\ell}^{(t)}$ the policy parameter of agent $i$ at the $\ell$-th local step of the $t$-th communication round. At the beginning of each round $t$, $\theta_{i,0}^{(t)}$ is initialized from a common global policy parameter $\bar{\theta}^{(t)}$. To update $\theta_{i,\ell}^{(t)}$, agent $i$ first samples a truncated trajectory of length $K$ by playing the parameterized policy $\pi_{\theta_{i,\ell}^{(t)}}$ in its own MDP $\mc{M}_i$. Doing so enables agent $i$ to compute the noisy truncated gradient $\hat\nabla_K J_i(\theta_{i,\ell}^{(t)})$ as per Eq.~\eqref{eqn:trunc_gradient}. The key question is: \emph{How should agent $i$ use $\hat\nabla_K J_i(\theta_{i,\ell}^{(t)})$ to update $\theta_{i,\ell}^{(t)}$?} Inspired by the popular FL algorithm \texttt{FedAvg}~\cite{mcmahan}, one natural strategy could be to use the following update: $\theta_{i,\ell+1}^{(t)}=\theta_{i,\ell}^{(t)} - \eta \hat{\nabla}_K J_i(\theta_{i,\ell}^{(t)})$. Running this update for several local steps will however cause agent $i$ to drift towards a locally optimal parameter $\theta^*_i$. This bias is undesirable since our goal is to instead converge to $\theta^*$ - a minimizer of the global value function $J(\theta)$ in Eq.~\eqref{eqn:problem}. We now describe our strategy for achieving this. 

\textbf{De-biasing/Drift mitigation.} We start by observing that if the agents could in fact communicate at all time-steps, they would ideally like to implement the update rule: $\bar{\theta}^{(t+1)} = \bar{\theta}^{(t)} - \eta \hat{\nabla}_K J(\bar{\theta}^{(t)})$, where $\hat{\nabla}_K J(\theta) \triangleq (1/N) \sum_{i \in [N]} \hat{\nabla}_K J_i(\theta)$. This is not possible however, since an agent $i$ cannot access the policy gradients of the other agents between communication rounds. The main idea behind our approach is to equip each agent with the memory of the global policy gradient direction $\hat{\nabla}_KJ(\Bar{\theta}^{(t)})$ from the beginning of the round. As an agent $i$ keeps interacting with its own MDP $\mc{M}_i$, however, its local policy parameter $\theta_{i,\ell}^{(t)}$ evolves from its value $\Bar{\theta}^{(t)}$ at the beginning of the round. To account for this staleness, agent $i$ adds the correction term $\hat{\nabla}_K J_i(\theta_{i,\ell}^{(t)})-\hat{\nabla}_K J_i(\Bar{\theta}^{(t)})$ to the global PG guiding direction $\hat{\nabla}_KJ(\Bar{\theta}^{(t)})$. This leads to the update rule for \texttt{Fast-FedPG}:
\begin{equation}
\theta_{i,\ell+1}^{(t)}=\theta_{i,\ell}^{(t)}- \eta \Bigl(\hat{\nabla}_K J_i(\theta_{i,\ell}^{(t)})-\hat{\nabla}_K J_i(\Bar{\theta}^{(t)})+\hat{\nabla}_KJ(\Bar{\theta}^{(t)})\Bigr).\label{update_FAPG}
\end{equation}
At the end of $H$ local steps, the agents transmit the change in their local parameters over the entire round to the server (line 9). The server then updates the global parameter as
\begin{equation}
\label{eqn:serv_aggr}
    \bar\theta^{(t+1)}=\bar\theta^{(t)}+\frac{\alpha_g}{N}\sum_{i=1}^N \Delta_{i,H}^{(t)},
\end{equation}
where $\alpha_g \in (0, 1]$ is a global step-size. We note here that while drift-mitigation strategies similar to the one above have been studied in federated supervised learning~\cite{mitraNIPS21, scaffold}, it is unclear a priori whether they can yield fast rates for our RL setting. In particular, the lack of convexity and the use of noisy truncated policy gradients (in Eq.~\eqref{update_FAPG}) that are inherently biased leads to unique challenges in analyzing the dynamics of \texttt{Fast-FedPG}. Despite such challenges, we will provide a rigorous convergence analysis of \texttt{Fast-FedPG} in Section~\ref{sec:analysis}. 

\section{Main Results}\label{results}
\subsection{A key structural result}
We start by establishing an important result that will serve as the key enabler for achieving fast convergence rates. To motivate the need for this result, we note that in the context of empirical risk minimization for supervised learning, one typically relies on strong-convexity of the loss function to achieve linear convergence rates. Despite the non-convexity of the policy optimization landscape, some recent work~\cite{meiglobal, fazel} have shown that fast linear convergence to a globally optimal policy is still possible under a weaker (relative to strong-convexity) gradient-domination condition. This condition, however, depends on the policy parameterization and the properties of the \textit{underlying MDP}. In our case, since we care about convergence to $\theta^* \in \argmin_{\theta \in \mathbb{R}^d} J(\theta)$, a gradient-domination condition on the global policy gradient $\nabla J(\theta) \triangleq  (1/N) \sum_{i \in [N]} \nabla J_i (\theta)$ is required to achieve linear convergence to $\theta^*$. For this to happen, however, we need to link $\nabla J(\theta)$ to the policy gradient of some underlying MDP. 

Given this reasoning, the subject of this section is to construct an ``Average MDP" using the agents' MDPs, and establish that the PG of this average MDP is precisely equal to $\nabla J(\theta)$. Once this is achieved, a gradient-domination condition for the average MDP will immediately imply one for $\nabla J(\theta)$. With this in mind, we construct the average MDP as $\bar{\cM}=(\cS, \cA, \bar{R}, P, \gamma)$, where $\bar{R}(s,a) \triangleq \frac{1}{N}\sum_{i=1}^N R_i(s,a), \forall (s,a) \in \mc{S} \times \mc{A}.$ Similar to Eq.~\eqref{eqn:value_func}, we can define the value-function of this MDP for a policy $\pi_{\theta}$ as $\bar J(\theta) = \E\bracket{\sum_{t=0}^\infty \gamma^t \bar r^{(t)}\Big|s^{(0)}\sim\rho,\pi_\theta}$, where $\bar r^{(t)}=\bar R( s^{(t)},  a^{(t)})$. We then claim the following. 

\begin{proposition}
\label{prop:key_result}
    For any fixed policy $\pi_\theta$ and initial distribution $\rho$, we have $\nabla \bar J(\theta) = \nabla J(\theta) = \frac{1}{N}\sum_{i=1}^N \nabla J_i(\theta),$ where $\nabla \bar J(\theta)$ is the gradient (w.r.t. $\theta$) of the value-function $\bar{J}(\theta)$ corresponding to the average MDP $\bar\cM$.
\end{proposition}
\begin{proof}
We will prove this result in three steps by making some simple observations. To proceed, {let us use the notation $\texttt{Avg}(\{c_i\}) \triangleq (1/N) \sum_{i\in [N]} c_i$ to denote the average of $N$ scalars $c_1, \ldots, c_N$.}  

\textbf{Step 1.} Define $R_i^{\pi_\theta}(s) \triangleq \sum_{a\in\cA}R_i(s, a)\pi_\theta(a|s)$. For any fixed policy $\pi_\theta$, we then claim that $\bar R^{\pi_\theta}(s) = \texttt{Avg}(\{ R_i^{\pi_\theta}(s)\}), \forall s \in \cS.$ In words, this simply states that the reward function induced by a policy $\pi_\theta$ on the average MDP $\bar{\mc{M}}$ is the average of the reward functions induced by the same policy on the agents' MDPs. To see this, observe:
\begin{equation}
            \bar R^{\pi_\theta}(s) = \sum_{a\in\cA}\bar{R}(s, a)\pi_\theta(a|s)= \sum_{a\in\cA}\frac{1}{N}\sum_{i=1}^NR_i(s, a)\pi_\theta(a|s).
\nonumber
    \end{equation}
The claim then follows by swapping the order of the summation, and using the definition of $R_i^{\pi_\theta}(s)$. Before we present the next fact, with a slight overload of notation, let us use $J_i(\theta,s)$ to represent the value-function $J_i(\theta)$ when the initial state is $s\in \mc{S}$ deterministically. We can define $\bar{J}(\theta,s)$ accordingly. Next, define the state-action value function as $Q_i^{\pi_\theta}(s,a)=\E\bracket{\sum_{t=0}^\infty \gamma^t r_{i}^{(t)}\Big|s^{(0)}=s,a^{(0)}=a,\pi_\theta}$. 

\textbf{Step 2.} For any fixed policy $\pi_\theta$, we claim
(i) $\bar{J}(\theta, s) = \texttt{Avg}(\{J_i(\theta,s)\}), \forall s \in \mc{S}$, and (ii) $\bar Q^{\pi_\theta}(s,a) = \texttt{Avg}(\{Q_i^{\pi_\theta}(s,a)\}),$ where $\bar{J}(\theta, s)$ and $\bar Q^{\pi_\theta}(s,a)$ are the value-function and state-action value function induced by the policy $\pi_\theta$ on the average MDP $\bar{\mc{M}}.$ To prove this claim, we will exploit the fact that the policy $\pi_{\theta}$ induces the same Markov transition matrix $\mathbf{P}^{\theta}$ on each MDP $\mc{M}_i, i \in [N]$, as well as on $\bar{\mc{M}}$, since they all share the same transition kernels. From the policy-specific Bellman fixed-point equation~\cite{puterman}, we then have:
    \begin{equation}
            \mathbf{J}_i^{\theta} = (\mb{I}-\gamma\mb{P}^{\theta})^{-1} \mb{R}_i^{\theta}, \forall i\in[N], \mb{\bar{J}}^{\theta} = (\mb{I}-\gamma\mb{P}^{\theta})^{-1} \mb{\bar{R}}^{\theta},
\label{eqn:Bellman_fxdpt}
    \end{equation}
where we stacked up $R^{\pi_{\theta}}_i(s), \bar{R}^{\pi_{\theta}}(s), J_i(\theta, s)$, and $\bar{J}(\theta,s)$ for different states into the vectors $\mb{R}^{\theta}_i, \mb{\bar{R}}^{\theta}$, $\mathbf{J}_i^{\theta}$, and $\mb{\bar{J}}^{\theta}$. The claim that $\bar{J}(\theta, s) = \texttt{Avg}(\{J_i(\theta,s)\}), \forall s \in \mc{S}$, then immediately follows from Eq.~\eqref{eqn:Bellman_fxdpt} and Step 1. Next, observe 
    \begin{equation}
        \begin{aligned}
            \bar Q^{\pi_\theta}(s,a)&=\bar R(s,a)+\gamma\E_{s'\sim \cP(\cdot| s,a)}\bracket{\bar{J}(\theta, s')}\\
            &\overset{(a)}{=}\frac{1}{N}\sum_{i=1}^NR_i(s,a)+\gamma\E_{s'\sim \cP(\cdot| s,a)}\bracket{\frac{1}{N}\sum_{i=1}^N J_i(\theta, s')}\\
            &=\frac{1}{N}\sum_{i=1}^N\left( \underbrace{R_i(s,a)+\gamma\E_{s'\sim \cP(\cdot| s,a)}\bracket{ J_i(\theta, s')}}_{Q_i^{\pi_\theta}(s,a)}\right),
        \end{aligned}
\nonumber 
    \end{equation}
    where for $(a)$, we used $\bar{J}(\theta, s) = \texttt{Avg}(\{J_i(\theta,s)\})$. 
    
\textbf{Step 3.} To complete the last step, recall the definition of \emph{state occupancy measure} from~\cite{agarwal2019reinforcement}:
$$ 
d^{\pi_{\theta}}_{s^{(0)}}(s) = (1-\gamma) \sum_{t=0}^{\infty} \gamma^t \mathbb{P}(s_t = s| s^{(0)}, \pi_\theta),$$
where $\mathbb{P}(s_t = s| s^{(0)}, \pi)$ denotes the probability of starting from $s^{(0)}$ and ending up in $s$ at round $t$ by playing policy $\pi_\theta$. From \cite[Theorem 11.4]{agarwal2019reinforcement}, we then know that 
\begin{equation}
\nabla J_i(\theta) = \frac{1}{1-\gamma} \E_{s \sim d^{\pi_{\theta}}_{\rho}} \bracket{\sum_{a\in\cA}\nabla \log \pi_\theta (a|s) Q_i^{\pi_\theta}(s,a)\pi_\theta(a|s)},
\label{eqn:grad_expression}
\end{equation}
where $d_\rho^{\pi_\theta}(s)=\E_{s^{(0)}\sim\rho}\bracket{d_{s^{(0)}}^{\pi_\theta}(s)}$. For the average MDP $\bar{\mc{M}}_i$, $\nabla \bar{J}(\theta)$ can be computed exactly as in Eq.~\eqref{eqn:grad_expression}, with just $Q_i^{\pi_\theta}(s,a)$ replaced by $\bar{Q}^{\pi_\theta}(s,a)$. This is because identical transition kernels imply identical occupancy measures across the agents' MDPs and the average MDP. The claim in Proposition~\ref{prop:key_result} then follows immediately from Step 2 where we showed that $\bar Q^{\pi_\theta}(s,a) = \texttt{Avg}(\{Q_i^{\pi_\theta}(s,a)\}).$
\end{proof} 

\subsection{Assumptions and main convergence results}
To obtain our main results, we need to make a few standard assumptions that we state and describe below. 
\begin{assumption}[Smoothness]\label{ass:smoothness} There exists a constant $L \geq 1$ such that for each agent $i\in[N]$,  $J_i(\cdot)$ is $L$-smooth, i.e., 
    \begin{align}
        \norm{\nabla J_i(\theta_1)-\nabla J_i(\theta_2)}\leq L\norm{\theta_1-\theta_2}, \forall  \theta_1,\theta_2 \in \mathbb{R}^d,
    \nonumber
    \end{align}
    where $\nabla J_i(\cdot)$ is the exact gradient of $J_i(\cdot)$ as defined in (\ref{exact_gradient}).\footnote{{Unless otherwise specified, we will use $\Vert \cdot \Vert$ to denote the Euclidean norm.}}
\end{assumption}
The smoothness of local objective functions immediately implies that of the global objective function, yielding: 
    \begin{align}
        \norm{\nabla J(\theta_1)-\nabla J(\theta_2)}\leq L\norm{\theta_1-\theta_2}, \forall  \theta_1,\theta_2 \in \mathbb{R}^d. 
    \nonumber
    \end{align}

Almost all papers on PG methods we are aware of rely on smoothness~\cite{meiglobal, agarwalPG, yuan2022, xiao2022}. The next assumption follows directly from the definition of $\nabla_K J_i(\cdot)$. 

\begin{assumption}[Unbiasedness]\label{ass:unbiasedness}
    For each agent $i\in[N]$,  $\hat{\nabla}_K J_i(\cdot)$ is an unbiased estimate of  $\nabla_K J_i(\cdot)$. 
\end{assumption}

Next, we make a bounded variance assumption that is typical in the literature on stochastic optimization.

\begin{assumption}[Bounded variance]\label{ass:variance} There exists a constant $\sigma \geq 1$ such that 
    \begin{align}
        \E\bracket{\norm{\hat{\nabla}_K J_i(\theta)-{\nabla}_K J_i(\theta)}^2}\leq \sigma^2, \forall i \in [N], \forall \theta \in \mathbb{R}^d. 
\nonumber
    \end{align}
\end{assumption}

The term $\sigma$ captures the variance in the noisy gradients. Our next assumption will help to control the effect of truncating the gradients~\cite{yuan2022}.

\begin{assumption}[Truncation]\label{ass:truncation}
There exists a constant $D \geq 1$ such that for each agent $i\in[N]$, the following bound holds: 
    \begin{align}
        \norm{{{\nabla}_K J_i(\theta)}-\nabla J_i(\theta)}&\leq D\gamma^K, \forall \theta \in \mathbb{R}^d. 
    \end{align} 
\end{assumption}

Finally, we will assume that the trajectories across agents are statistically independent, as is done in FRL~\cite{FRL_identical_linear, fedtd, jin2022federated}. 

\begin{assumption}[Independence]\label{ass:independence}
    We assume that the sampled trajectories $\tau_{i}, i\in[N]$ are independent across agents.
\end{assumption}

Given the above assumptions, our first main result characterizes \texttt{Fast-FedPG}'s progress in each round. 

\begin{theorem}\label{thm:onestep}
    Suppose Assumptions \ref{ass:smoothness} - \ref{ass:independence} hold. Define $\alpha=H\eta\alpha_g$ as the effective step-size. Then there exists a universal constant $C\geq 1$, such that with $\alpha_g=1$ and $\eta$ chosen to satisfy $\eta \leq 1/(4CLH)$, \texttt{Fast-FedPG} guarantees $\forall t \geq 0$:
\begin{equation}
\begin{aligned}
    \quad \E\bracket{J(\bar{\theta}^{(t+1)})} &\leq \E\bracket{J(\bar{\theta}^{(t)})}- \frac{\alpha}{4}\E\bracket{\norm{\nabla J(\Bar{\theta}^{(t)})}^2}+ \bigo{\frac{\alpha^2 L\sigma^2}{NH}+\alpha^3L^2\sigma^2} + \bigo{\alpha}D^2\gamma^{2K}.\label{eqn:thm:onestep}
\end{aligned}
\end{equation}
\end{theorem}

We prove Theorem~\ref{thm:onestep} in Section~\ref{sec:analysis}. For now, let us see how Theorem~\ref{thm:onestep} yields fast rates under gradient-domination. 

\begin{theorem}\label{thm:fastrate} (\textbf{Fast rates}) 
    Suppose all the conditions in Theorem \ref{thm:onestep} hold. Additionally, suppose the following gradient-domination condition is satisfied by the average MDP:
\begin{equation}
 \mu (\bar{J}(\theta) - \bar{J}(\theta^*)) \leq \norm{\nabla \bar{J}(\theta)}^2, \forall \theta \in \mathbb{R}^d,
\label{eqn:grad_dom}
\end{equation}
 for some $\mu >0.$ Then, \texttt{Fast-FedPG} guarantees $\forall T \geq 0$:
    \begin{equation}
\label{eqn:fast_rate}
        \begin{aligned}
            \E\bracket{J(\bar{\theta}^{(T)})-J({\theta^*})}&\leq \left(1-\frac{\alpha\mu}{4}\right)^T\left({J(\bar{\theta}^{(0)})-J({\theta^*})}\right)+\bigo{\frac{\alpha L\sigma^2}{\mu NH}+\frac{\alpha^2L^2\sigma^2}{\mu}} + \bigo{\frac{D^2\gamma^{2K}}{\mu}}.
        \end{aligned}
    \end{equation}
\end{theorem}
\begin{proof}
The statement and proof of Proposition~\ref{prop:key_result} tell us that $\nabla \bar{J}(\theta) =\nabla J(\theta)$ and $\bar{J}(\theta) = J(\theta), \forall \theta \in \mathbb{R}^d$.  Combining this with Eq.~\eqref{eqn:grad_dom}, we get $\mu ({J}(\theta) - {J}(\theta^*)) \leq \norm{\nabla {J}(\theta)}^2, \forall \theta \in \mathbb{R}^d.$ Plugging this bound into Eq.~\eqref{eqn:thm:onestep} and unrolling the resulting inequality leads to the desired claim.
\end{proof}

\noindent \textbf{Discussion.} To parse Theorem~\ref{thm:fastrate}, we note that in the absence of noise (i.e., $\sigma=0$) and truncation errors (i.e., $D=0$), \texttt{Fast-FedPG} guarantees linear convergence of $J(\bar{\theta}^{(T)})$ to the globally optimal value $J(\theta^*)$. This is consistent with recent findings in the centralized PG literature~\cite{meiglobal, yuan2022} that achieve similar linear rates under gradient-domination. 

\textbf{Linear speedup.} We now discuss how under a suitable selection of the local step-size $\eta$, the number of communication rounds $T$, and the roll-out horizon $K$, one can achieve a linear speedup result from Theorem~\ref{thm:fastrate}. To that end, suppose 
$$ \eta = \frac{4}{\mu H} \frac{\log(NHT)}{T}, \quad T \geq \frac{L}{\mu} \max\{ 16 C \log(NHT), NH \}.$$
Note that $T$ can always be chosen large enough to meet the above condition, and the above choices of $\eta$ and $T$ respect the criterion $\eta \leq 1/(4CLH)$ needed for Theorem~\ref{thm:onestep} to hold. Next, let the roll-out horizon $K$ be picked to satisfy:
$ K \geq {\log (NHT)}/(2\log (1/\gamma)).$
Substituting the above choices of parameters into Eq.~\eqref{eqn:fast_rate}, and simplifying, we obtain:
\begin{equation}
\E\bracket{J(\bar{\theta}^{(T)})-J({\theta^*})} \leq \tilde{O}\left(\left( G + \frac{L\sigma^2}{\mu^2} + \frac{D^2}{\mu} \right) \frac{1}{NHT}\right), 
\nonumber
\end{equation}
where $G=\left({J(\bar{\theta}^{(0)})-J({\theta^*})}\right)$. We note that despite noisy, biased policy gradients and reward-heterogeneity, \texttt{Fast-FedPG} guarantees convergence (in expectation) to a globally optimal policy parameter $\theta^*$ at the rate $\tilde{O}(1/(NHT))$. There are two important takeaways here. First, unlike~\cite{fedtd} and~\cite{jin2022federated}, our final rate exhibits no heterogeneity-induced bias. Second, the $\tilde{O}(1/(NHT))$ rate is essentially the best one can hope for since the total amount of data (i.e., trajectories) across agents over $T$ rounds is precisely $NHT$. Notably, our results clearly exhibit an $N$-fold speedup w.r.t. the number of agents (relative to the centralized setting), demonstrating the benefit of federation. These results are the first of their kind in the context of multi-task/federated policy gradients, and significantly improve upon those in \cite{xie2023fedkl} that only come with asymptotic rates, and those in \cite{zeng2021decentralized} that exhibit no linear speedup. 

Finally, suppose the gradient-domination condition no longer holds. Moreover, suppose the transition kernels across the agents are potentially non-identical. The proof of Theorem~\ref{thm:onestep} in Section~\ref{sec:analysis} reveals that Theorem~\ref{thm:onestep} continues to hold. An immediate consequence of this result is the following guarantee on convergence to a first-order stationary point. 

\begin{theorem}\label{thm:statpt}
    Suppose all the conditions in Theorem \ref{thm:onestep} hold. Then, \texttt{Fast-FedPG} guarantees:
    \begin{equation}
    \begin{aligned}
        \frac{1}{T}\sum_{t=0}^{T-1}\E\bracket{\norm{\nabla J(\Bar{\theta}^{(t)})}^2} &\leq \frac{4\E\bracket{J(\bar{\theta}^{(0)})-J(\bar{\theta}^{(T)})}}{\alpha T}+\bigo{\frac{\alpha L\sigma^2}{NH}+\alpha^2L^2\sigma^2} + \bigo{D^2\gamma^{2K}}.\label{eqn:cor:statpt}
    \end{aligned}
\nonumber 
    \end{equation}
\end{theorem}
With $\eta=\frac{4}{H}\sqrt{\frac{NH}{T}}, \quad T\geq L^2 \max\{256C^2NH, N^3H^3\},$
and $K$ chosen as before, we obtain a final convergence rate of $\tilde{O}(1/\sqrt{NHT})$ in this case. 
Once again, there is a clear benefit of collaboration captured by the inverse scaling of this bound w.r.t. $\sqrt{N}$. 
\section{Analysis}\label{sec:analysis} 
The goal of this section is to provide a detailed convergence proof of Theorem~\ref{thm:onestep}. To that end, we begin by noting that if a function $f: \mathbb{R}^d \mapsto \mathbb{R}$ is $L$-smooth, then the following holds for any two points $x,y \in \mathbb{R}^d$:  
\begin{equation}
    f(y)-f(x) \leq \langle y-x, \nabla f(x) \rangle +\frac{L}{2}{\Vert y-x \Vert}^2.
\label{eqn:smooth}
\end{equation}

We will also make use of the following two elementary facts at various points of our analysis.
\begin{itemize}
    \item Given any two vectors $x,y\in\mathbb{R}^d$, for any $\xi >0$:
\begin{equation}
    {\Vert x+y \Vert}^2 \leq (1+\xi){\Vert x \Vert}^2 + \left(1+\frac{1}{\xi}\right){\Vert y \Vert}^2.
\label{eqn:rel_triangle}
\end{equation}
 \item Given $m$ vectors $x_1,\ldots,x_m\in\mathbb{R}^d$, the following is a simple application of Jensen's inequality:
 \begin{equation}
     \norm{ \sum\limits_{i=1}^{m} x_i}^2 \leq m \sum\limits_{i=1}^{m} {\Vert x_i \Vert}^2.
\label{eqn:Jensens}
 \end{equation}
\end{itemize}
\textbf{Sketch of proof.} Before delving into the technical details, let us first provide an overview of our analysis. Our first main step is to exploit smoothness of the local objective functions to establish a one-round progress bound for \texttt{Fast-FedPG}. 
\begin{lemma} \label{lemm:one-step} Suppose Assumptions \ref{ass:smoothness} - \ref{ass:independence} hold. Let $\Delta_{i,\ell}^{(t)} = \theta^{(t)}_{i,\ell} - \bar{\theta}^{(t)}.$ Then, the following is true for \texttt{Fast-FedPG}: 
\begin{equation}
\label{eqn:onestepbnd}
\begin{aligned}
    \quad \E\bracket{J(\bar{\theta}^{(t+1)})} &\leq \E\bracket{J(\bar{\theta}^{(t)})} - \frac{\alpha}{2}\left(1 -8 \alpha L \right)\E\bracket{\norm{\nabla J(\Bar{\theta}^{(t)})}^2}  \\
    &\quad + \alpha L \left(\frac{L+ 4 \alpha L^2}{NH}\right) \sum_{i=1}^N\sum_{\ell=0}^{H-1}\E\bracket{\norm{\Delta_{i,\ell}^{(t)}}^2} + \left(\alpha + 2 \alpha^2 L\right)D^2\gamma^{2K} + \frac{2\alpha^2 L\sigma^2}{NH}.
\end{aligned}
\end{equation}
\end{lemma}

The above lemma relates the progress made in a particular round $t$ to the magnitude of the policy gradient $\norm{\nabla J(\Bar{\theta}^{(t)})}$. Notably, the progress is not controlled by the policy gradients of the agents' individual MDPs, but rather by the policy gradient of the global objective function. This is precisely what we want to ensure that progress is made towards $\theta^*$, not $\theta^*_i$ for any agent $i$. The object that impedes the progress is the client-drift term $\sum_{i=1}^N \sum_{\ell=0}^{H-1} \E\bracket{\norm{\Delta_{i,\ell}^{(t)}}^2}$. Therefore, it is apparent that to further refine the bound in Eq.~\eqref{eqn:onestepbnd}, we need to control this drift effect. To that end, we have the following lemma. 

\begin{lemma}
\label{lemma:drift}
Suppose Assumptions~\ref{ass:smoothness} - \ref{ass:truncation} hold. Let the local step-size $\eta$ satisfy $3 \eta L H \leq 1.$ Then, the following holds for the expected client-drift $\forall i \in [N], \forall \ell \in \{0,\cdots,H-1\}$:
\begin{equation}
\E\bracket{\norm{\Delta_{i,\ell}^{(t)}}^2} \leq 32 \eta^2 H^2 \underbrace{\left(\E\bracket{\norm{{\nabla}J(\bar{\theta}^{(t)})}^2} + 18 \sigma^2 + 18 D^2 \gamma^{2K}\right)}_{\mc{G}^{(t)}}.
\label{eqn:final_drift_bnd}
\end{equation}
\end{lemma}

To gain some intuition about the above result, suppose that there is no noise, i.e., $\sigma =0$, and no truncation, i.e., $D=0$. In other words, suppose all policy gradients are exact. Lemma~\ref{lemma:drift} then tells us that the drift over the round $t$ is caused due to an $O(\eta^2 H^2 \norm{{\nabla}J(\bar{\theta}^{(t)})}^2)$ perturbation. We immediately observe that if $\bar{\theta}^{(t)}= \theta^*$, i.e., the parameter at the beginning of the round is where we eventually want it to be,  then there will be no drift. This is again precisely what we desire, and aligns with the design strategy behind our algorithm \texttt{Fast-FedPG}. 

To summarize the discussion, up to noise- and truncation-induced errors, the ``good" term that contributes to progress is on the order of $\alpha \norm{{\nabla}J(\bar{\theta}^{(t)})}^2$, while the ``bad" term that impedes progress is $O(\eta^2 H^2 \norm{{\nabla}J(\bar{\theta}^{(t)})}^2)$. Since the bad term is a higher-order term in the step-size, by tuning the local and global step-sizes appropriately, one can hope to achieve overall progress. Making the above informal argument precise takes quite a bit of work. In what follows, we flesh out the details, starting with the proof of Lemma~\ref{lemm:one-step}. 
 
\begin{proof} (\textbf{Proof of Lemma~\ref{lemm:one-step}}) From the update rule of \texttt{Fast-FedPG} in Eq.~\eqref{update_FAPG}, we obtain

\begin{align}
    \Bar{\theta}^{(t+1)}-\Bar{\theta}^{(t)} & = -\frac{\alpha_g\eta}{N}\sum_{i=1}^N\sum_{\ell=0}^{H-1}\Bigl(\hat{\nabla}_K J_i(\theta_{i,\ell}^{(t)})-\hat{\nabla}_K J_i(\Bar{\theta}^{(t)})+\hat{\nabla}_KJ(\Bar{\theta}^{(t)})\Bigr)\notag\\
    & =-\frac{\alpha_g\eta}{N}\sum_{i=1}^N\sum_{\ell=0}^{H-1} \hat{\nabla}_K J_i(\theta_{i,\ell}^{(t)}). \notag
\end{align}

Recalling $\alpha=H\eta\alpha_g$, and using Eq.~\eqref{eqn:smooth} in view of the fact that $J(\theta)$ is smooth, we then obtain 
\begin{align}
 \label{eqn:one_step} 
    \E\bracket{J(\Bar{\theta}^{(t+1)})} & \leq \E\bracket{J(\Bar{\theta}^{(t)})} + \E\bracket{\iprod{\nabla J(\bar{\theta}^{(t)})}{\Bar{\theta}^{(t+1)}-\Bar{\theta}^{(t)}}}  +\frac{L}{2}\E\bracket{\norm{\Bar{\theta}^{(t+1)}-\Bar{\theta}^{(t)}}^2}  \notag\\
    &\overset{(a)}{=} \E\bracket{J(\Bar{\theta}^{(t)})} -\E\bracket{\iprod{\nabla J(\bar{\theta}^{(t)})}{\frac{\alpha}{NH}\sum_{i=1}^N\sum_{\ell=0}^{H-1} {\nabla}_K J_i(\theta_{i,\ell}^{(t)})}} \notag\\ 
    &\ + \frac{L}{2}\E\bracket{\norm{\frac{\alpha}{NH}\sum_{i=1}^N \sum_{\ell=0}^{H-1}\hat{\nabla}_K J_i(\theta_{i,\ell}^{(t)})}^2}\notag \\
    &= \E\bracket{J(\Bar{\theta}^{(t)})} \underbrace{-\E\bracket{\iprod{\nabla J(\bar{\theta}^{(t)})}{\frac{\alpha}{NH}\sum_{i=1}^N\sum_{\ell=0}^{H-1} {\nabla} J_i(\theta_{i,\ell}^{(t)})}}}_{T_1} \notag\\
    &\ +\underbrace{\E\bracket{\iprod{\nabla J(\bar{\theta}^{(t)})}{\frac{\alpha}{NH}\sum_{i=1}^N\sum_{\ell=0}^{H-1} \left({\nabla} J_i(\theta_{i,\ell}^{(t)}) - {\nabla}_K J_i(\theta_{i,\ell}^{(t)})\right)}}}_{T_2} \notag\\
    &\ + \frac{L}{2}\E\Biggl[\Biggl\|\frac{\alpha}{NH}\sum_{i=1}^N \sum_{\ell=0}^{H-1}\Bigl({\nabla} J_i(\theta_{i,\ell}^{(t)})+\hat{\nabla}_K J_i(\theta_{i,\ell}^{(t)})-{\nabla}_K J_i(\theta_{i,\ell}^{(t)}) + {\nabla}_K J_i(\theta_{i,\ell}^{(t)})-{\nabla} J_i(\theta_{i,\ell}^{(t)})\Bigr)\Biggr\|^2\Biggr]\notag \\
    &\overset{(b)}{\leq} \E\bracket{J(\Bar{\theta}^{(t)})} +T_1+T_2+\underbrace{2{L}\E\bracket{\norm{\frac{\alpha}{NH}\sum_{i=1}^N \sum_{\ell=0}^{H-1}{\nabla} J_i(\theta_{i,\ell}^{(t)})}^2}}_{T_3}\notag\\
    &\ +\underbrace{2{L}\E\bracket{\norm{\frac{\alpha}{NH}\sum_{i=1}^N \sum_{\ell=0}^{H-1}\left(\hat{\nabla}_K J_i(\theta_{i,\ell}^{(t)})-{\nabla}_K J_i(\theta_{i,\ell}^{(t)})\right)}^2}}_{T_4}\notag\\
    &\ +\underbrace{2{L}\E\bracket{\norm{\frac{\alpha}{NH}\sum_{i=1}^N \sum_{\ell=0}^{H-1}\left({\nabla}_K J_i(\theta_{i,\ell}^{(t)})-{\nabla} J_i(\theta_{i,\ell}^{(t)})\right)}^2}}_{T_5},   
\end{align}
where (a) follows from the unbiasedness assumption (i.e., Assumption~\ref{ass:unbiasedness}) in tandem with the tower property of expectations, and (b) uses Eq.~\eqref{eqn:Jensens}. Let us now proceed to bound each of the five terms $T_1, T_2, T_3, T_4$, and $T_5$. For the term $T_1$, by defining $v_{i,\ell}^{(t)} \triangleq \nabla J_i(\theta_{i,\ell}^{(t)})-\nabla J_i(\bar{\theta}^{(t)})$ and recalling $\Delta_{i,\ell}^{(t)} = \theta^{(t)}_{i,\ell} - \bar{\theta}^{(t)}$, we have
\begin{equation}
\begin{aligned}
    {T_1} &= -\frac{\alpha}{NH}\sum_{i=1}^N\sum_{\ell=0}^{H-1}\E\bracket{\iprod{\nabla J(\Bar{\theta}^{(t)})}{ v_{i,\ell}^{(t)}}} -\frac{\alpha}{NH}\sum_{i=1}^N\sum_{\ell=0}^{H-1}\E\bracket{\iprod{\nabla J(\Bar{\theta}^{(t)})}{ \nabla J_i(\bar{\theta}^{(t)})}}  \\
    & = -\frac{\alpha}{NH}\sum_{i=1}^N\sum_{\ell=0}^{H-1}\E\bracket{\iprod{\nabla J(\Bar{\theta}^{(t)})}{v_{i,\ell}^{(t)}}} -\alpha \E\bracket{\norm{\nabla J(\Bar{\theta}^{(t)})}^2}  \\
    & \overset{(a)}{\leq} \frac{\alpha}{NH}\sum_{i=1}^N\sum_{\ell=0}^{H-1} \E\bracket{\frac{1}{4}\norm{\nabla J(\Bar{\theta}^{(t)})}^2 + \norm{v_{i,\ell}^{(t)}}^2} - \alpha\E\bracket{\norm{\nabla J(\Bar{\theta}^{(t)})}^2}  \\
    & \overset{(b)}\leq \frac{\alpha}{NH}\sum_{i=1}^N \sum_{\ell=0}^{H-1}\E\bracket{\frac{1}{4}\norm{\nabla J(\Bar{\theta}^{(t)})}^2 + L^2\norm{\Delta_{i,\ell}^{(t)}}^2} - \alpha \E\bracket{\norm{\nabla J(\Bar{\theta}^{(t)})}^2}  \\
    & = \frac{\alpha L^2}{NH}\sum_{i=1}^N\sum_{\ell=0}^{H-1}\E\bracket{\norm{\Delta_{i,\ell}^{(t)}}^2} - \frac{3\alpha}{4}\E\bracket{\norm{\nabla J(\Bar{\theta}^{(t)})}^2},
\end{aligned} 
\end{equation}
where (a) is a result of Young's inequality, and (b) exploits the smoothness assumption, i.e., Assumption~\ref{ass:smoothness}. For the term $T_2$, defining $e_i(\theta) \triangleq {\nabla}_K J_i(\theta)-\nabla J_i(\theta)$ and using Young's inequality, we get
\begin{equation}
\begin{aligned}
    {T_2}&\leq\frac{\alpha}{NH}\sum_{i=1}^N\sum_{\ell=0}^{H-1}\E\bracket{\frac{1}{4}\norm{\nabla J(\Bar{\theta}^{(t)})}^2+\norm{e_i(\theta_{i,\ell}^{(t)})}^2}\\
    &\leq \frac{\alpha}{4}\E\bracket{\norm{\nabla J(\Bar{\theta}^{(t)})}^2}+\alpha D^2\gamma^{2K},
\end{aligned}
\end{equation}
where the last step uses the truncation error bound in Assumption~\ref{ass:truncation}. Next, we turn our attention to $T_3$, and observe 
\begin{align}
    T_3 &= 2{L} \E\bracket{\norm{\frac{\alpha}{NH}\sum_{i=1}^N\sum_{\ell=0}^{H-1}  \left(v_{i,\ell}^{(t)}+\nabla J_i(\bar{\theta}^{(t)})\right)}^2} \notag\\
    &= 2{L}\E\bracket{\norm{\frac{\alpha}{NH}\sum_{i=1}^N\sum_{\ell=0}^{H-1}  v_{i,\ell}^{(t)}+\alpha\nabla J(\bar{\theta}^{(t)})}^2} \notag\\
    & \leq 4{L}\E\bracket{\norm{\frac{\alpha}{NH}\sum_{i=1}^N \sum_{\ell=0}^{H-1} v_{i,\ell}^{(t)}}^2} + 4{\alpha^2 L}\E\bracket{\norm{\nabla J(\bar{\theta}^{(t)})}^2}\notag \\
    & \overset{(a)}{\leq}  {\frac{4 \alpha^2 L}{NH}}\sum_{i=1}^N\sum_{\ell=0}^{H-1}\E\bracket{\norm{v_{i,\ell}^{(t)}}^2}+4{\alpha^2 L}\E\bracket{\norm{\nabla J(\bar{\theta}^{(t)})}^2} \notag \\
    & \overset{(b)}{\leq} {\frac{4\alpha^2 L^3}{NH}}\sum_{i=1}^N\sum_{\ell=0}^{H-1}\E\bracket{\norm{\Delta_{i,\ell}^{(t)}}^2} +4{\alpha^2 L}\E\bracket{\norm{\nabla J(\bar{\theta}^{(t)})}^2},
\end{align}    
where (a) follows from Eq.~\eqref{eqn:Jensens} and (b) from smoothness. 
As for the term $T_4$, we claim: 
\begin{equation}
\begin{aligned}
     T_4 &\overset{(a)}={\frac{2\alpha^2 L}{N^2H^2}}\sum_{i=1}^N \sum_{\ell=0}^{H-1}\E\bracket{\norm{\hat{\nabla}_K J_i(\theta_{i,\ell}^{(t)})-{\nabla}_K J_i(\theta_{i,\ell}^{(t)})}^2}\\
    & \overset{(b)}\leq {\frac{2\alpha^2 L}{N^2H^2}}\sum_{i=1}^N \sum_{\ell=0}^{H-1} \sigma^2= {\frac{2\alpha^2 L \sigma^2}{NH}},
\end{aligned}
\end{equation}
where (b) follows from the variance bound in Assumption~\ref{ass:variance}. To see why (a) holds, define $\cF_\ell^{(t)}$ as the sigma-algebra that captures all the randomness up to the $\ell$-th local iteration of round $t$. We also define $h_i(\theta) \triangleq \hat{\nabla}_K J_i(\theta)-\nabla_K J_i(\theta)$. Given these definitions, expanding $T_4$ yields: 
\begin{equation}
    \begin{aligned}
        T_4/\left({\frac{2\alpha^2 L}{N^2H^2}}\right) &=\sum_{i=1}^N \sum_{\ell=0}^{H-1}\E\bracket{\norm{h_i(\theta_{i,\ell}^{(t)})}^2}+2\sum_{j<k}\sum_{m<n}\E\bracket{\E\bracket{\iprod{h_j(\theta_{j,m}^{(t)})}{h_k(\theta_{k,n}^{(t)})}\Big|\cF_{n-1}^{(t)}}}\\
        &\ +2\sum_{j=1}^N\sum_{m<n} \E\bracket{\E\bracket{\iprod{h_j(\theta_{j,m}^{(t)})}{h_j(\theta_{j,n}^{(t)})}\Big|\cF_{n-1}^{(t)}}} \notag\\
        &\ +2\sum_{j<k}\sum_{m=0}^{H-1}\E\bracket{\E\bracket{\iprod{h_j(\theta_{j,m}^{(t)})}{h_k(\theta_{k,m}^{(t)})}\Big|\cF_{m-1}^{(t)}}} \\
        & = \sum_{i=1}^N \sum_{\ell=0}^{H-1}\E\bracket{\norm{h_i(\theta_{i,\ell}^{(t)})}^2}.
    \end{aligned}
\end{equation}
In the third cross-term, by conditioning on $\cF_{m-1}^{(t)}$,  $\theta_{j,m}^{(t)}$ and $\theta_{k,m}^{(t)}$ become deterministic, and the only randomness comes from the noise in $\hat\nabla_K J_j(\theta_{j,m}^{(t)})$ and $\hat\nabla_K J_k(\theta_{k,m}^{(t)})$ due to sampling trajectories, and the trajectories across agents $j$ and $k$ are independent according to Assumption \ref{ass:independence}. We thus have $\E\bracket{\iprod{h_j(\theta_{j,m}^{(t)})}{h_k(\theta_{k,m}^{(t)})}\Big|\cF_{m-1}^{(t)}} = \iprod{\E\bracket{h_j(\theta_{j,m}^{(t)})\Big|\cF_{m-1}^{(t)}}}{\E\bracket{h_k(\theta_{k,m}^{(t)})\Big|\cF_{m-1}^{(t)}}}=0$, where in the last step we used that based on Assumption \ref{ass:unbiasedness}, $h_j(\theta_{j,m}^{(t)})$ and $h_k(\theta_{k,m}^{(t)})$ are zero-mean conditioned on $\cF_{m-1}^{(t)}$. The fact that the first two cross-terms vanish can be explained similarly. Finally, for the term $T_5$, we can use Eq.~\eqref{eqn:Jensens} followed by Assumption~\ref{ass:truncation} to obtain
\begin{equation}
   T_5 \leq {\frac{2 \alpha^2 L}{NH}}\sum_{i=1}^N \sum_{\ell=0}^{H-1}\norm{e_i(\theta_{i,\ell}^{(t)})}^2 \leq 2 \alpha^2 L D^2 \gamma^{2K}. 
\end{equation}

Plugging in the bounds we obtained for the terms $T_1$-$T_5$ in Eq.~\eqref{eqn:one_step} directly leads to the claim of the lemma. 
\end{proof}

We now turn to the proof of Lemma~\ref{lemma:drift}. 
\begin{proof} (\textbf{Proof of Lemma~\ref{lemma:drift}}) Our immediate goal is to use the update rule of \texttt{Fast-FedPG} to obtain a recursion for $\Delta_{i,\ell}^{(t)} = \theta^{(t)}_{i,\ell} - \bar{\theta}^{(t)}.$ To that end, let us first define a bit of notation as follows. Let $\cV_{i,\ell}^{(t)} \triangleq {\nabla}J_i(\theta_{i,\ell}^{(t)}) -{\nabla}J_i(\bar{\theta}^{(t)})+{\nabla}J(\bar{\theta}^{(t)})$, and $\cW_{i,\ell}^{(t)} \triangleq h_i(\theta_{i,\ell}^{(t)})-h_i(\bar{\theta}^{(t)})+h(\bar{\theta}^{(t)})+e_i(\theta_{i,\ell}^{(t)})-e_i(\bar{\theta}^{(t)})+e(\bar{\theta}^{(t)})$, where $h_i(\theta)=\hat{\nabla}_K J_i(\theta)-\nabla_K J_i(\theta)$, $h(\theta)=\frac{1}{N}\sum_{i=1}^Nh_i(\theta)$, $e_i(\theta) = {\nabla}_K J_i(\theta)-\nabla J_i(\theta)$ and $e(\theta)=\frac{1}{N}\sum_{i=1}^Ne_i(\theta)$. In words, $\cV_{i,\ell}^{(t)}$ is the ideal update direction of \texttt{Fast-FedPG} comprising of exact gradients, while $\cW_{i,\ell}^{(t)}$ captures deviations from this ideal update due to noise and truncation. From  Eq.~\eqref{update_FAPG}, we obtain
\begin{equation}
\begin{aligned}
    \norm{\Delta_{i,\ell+1}^{(t)}} &= \Bigl\|\Delta_{i,\ell}^{(t)} - \eta\Bigl(\hat{\nabla}_KJ_i(\theta_{i,\ell}^{(t)})-\hat{\nabla}_KJ_i(\bar{\theta}^{(t)})+\hat{\nabla}_KJ(\bar{\theta}^{(t)})\Bigr)\Bigr\| \\
    & = \Big\|\Delta_{i,\ell}^{(t)}- \eta\cV_{i,\ell}^{(t)}-\eta \cW_{i,\ell}^{(t)}\Big\| \\
    &\leq \norm{\Delta_{i,\ell}^{(t)}-\eta \cV_{i,\ell}^{(t)}} +\eta \norm{\cW_{i,\ell}^{(t)}}\\
    &\overset{(a)}{\leq} \norm{\Delta_{i,\ell}^{(t)}} +\eta \norm{v_{i,\ell}^{(t)}}+\eta\norm{{\nabla}J(\bar{\theta}^{(t)})} +\eta \norm{\cW_{i,\ell}^{(t)}} \\
    &\overset{(b)}{\leq} (1+\eta L)\norm{\Delta_{i,\ell}^{(t)}} +\eta\norm{{\nabla}J(\bar{\theta}^{(t)})} +\eta \norm{\cW_{i,\ell}^{(t)}},\label{delta1}
\end{aligned}
\end{equation}
where for (a), we used $v_{i,\ell}^{(t)} \triangleq \nabla J_i(\theta_{i,\ell}^{(t)})-\nabla J_i(\bar{\theta}^{(t)})$, and for (b), we used smoothness. Squaring both sides of the above display and using Eq.~\eqref{eqn:rel_triangle} with $\xi=H$, we obtain:
\begin{equation}
\label{eqn:driftbnd1}
\begin{aligned}
\norm{\Delta_{i,\ell+1}^{(t)}}^2 &\leq \left(1+\frac{1}{H}\right) (1+\eta L)^2 \norm{\Delta_{i,\ell}^{(t)}}^2 + (1+H) \eta^2 \left(\norm{{\nabla}J(\bar{\theta}^{(t)})}  +\norm{\cW_{i,\ell}^{(t)}}\right)^2 \\
& \leq \left(1+\frac{1}{H}\right) (1+3\eta L) \norm{\Delta_{i,\ell}^{(t)}}^2 + 4 \eta^2 H \left( \norm{{\nabla}J(\bar{\theta}^{(t)})}^2  +\norm{\cW_{i,\ell}^{(t)}}^2 \right),
\end{aligned}
\end{equation}
where in the second step, we used $H \geq 1$ and $\eta L \leq 1.$ Now from the way we defined $\cW_{i,\ell}^{(t)}$, observe that 
\begin{equation}
\label{eqn:driftbnd2}
\begin{aligned}
\norm{\cW_{i,\ell}^{(t)}}^2 &\leq 2 \underbrace{\norm{h_i(\theta_{i,\ell}^{(t)})-h_i(\bar{\theta}^{(t)})+h(\bar{\theta}^{(t)})}^2}_{(*)}+ 2 \norm{e_i(\theta_{i,\ell}^{(t)})-e_i(\bar{\theta}^{(t)})+e(\bar{\theta}^{(t)})}^2\\
& \hspace{-5mm} \overset{\eqref{eqn:Jensens}}{\leq} 2 (*) + 6 \norm{e_i(\theta_{i,\ell}^{(t)})}^2 + 6 \norm{ e_i(\bar{\theta}^{(t)})}^2 + 6 \norm{e(\bar{\theta}^{(t)})}^2 \\ 
& \hspace{-5mm} \leq  2 (*) + 18 D^2 \gamma^{2K},
\end{aligned}
\end{equation}
where in last step, we used the truncation error bound from Assumption~\ref{ass:truncation}. To establish our desired recursion for $\Delta_{i,\ell}^{(t)}$, it remains to control the term $(*)$ in Eq.~\eqref{eqn:driftbnd2}. This can be achieved using the variance bound from Assumption~\ref{ass:variance}:
\begin{equation}
\label{eqn:driftbnd3}
\begin{aligned}
\E\bracket{(*)} &\leq  3 \E\bracket{\norm{h_i(\theta_{i,\ell}^{(t)})}^2} + 3 \E\bracket{\norm{ h_i(\bar{\theta}^{(t)})}^2} + 3 \E\bracket{\norm{h(\bar{\theta}^{(t)})}^2}\\
& = 3 \E\bracket{\E\bracket{\norm{h_i(\theta_{i,\ell}^{(t)})}^2 \Big| \theta_{i,\ell}^{(t)}}} + 3 \E\bracket{\E\bracket{\norm{ h_i(\bar{\theta}^{(t)})}^2 \Big| \bar{\theta}^{(t)}}}+ 3 \E\bracket{\E\bracket{\norm{ h(\bar{\theta}^{(t)})}^2 \Big| \bar{\theta}^{(t)}}}\\ 
& \leq 6 \sigma^2 + \frac{3}{N^2} \E\bracket{ \sum_{i\in [N]} \E\bracket{\norm{ h_i(\bar{\theta}^{(t)})}^2 \Big| \bar{\theta}^{(t)}}} \leq 9 \sigma^2,\\
\end{aligned}
\end{equation}
where in the last step, we used Assumption~\ref{ass:variance}. Now taking expectations on both sides of Eq.~\eqref{eqn:driftbnd1}, and using the bounds from equations~\eqref{eqn:driftbnd2} and~\eqref{eqn:driftbnd3}, we obtain: 

\begin{equation}
\label{eqn:drift_recurs}
   \E\bracket{\norm{\Delta_{i,\ell+1}^{(t)}}^2} 
    \leq \underbrace{\left(1+\frac{1}{H}\right)(1+3\eta L)}_{\beta}\E\bracket{\norm{\Delta_{i,\ell}^{(t)}}^2} + 4 \eta^2 H \cG^{(t)}, 
\end{equation}
where $\mc{G}^{(t)}$ is as in~\eqref{eqn:final_drift_bnd}.
Now defining $d_{i,\ell} \triangleq  \E\bracket{\norm{\Delta_{i,\ell}^{(t)}}^2}$ and unrolling~\eqref{eqn:drift_recurs}, we obtain $\forall \ell \in [H-1]:$
\begin{equation}
\begin{aligned}
d_{i,\ell} &  \leq \beta^{\ell} d_{i,0} + 4\eta^2H\cG^{(t)} \left(\sum_{j=0}^{\ell-1} \beta^j\right)\\ 
& \overset{(a)}\leq \beta^H d_{i,0} + 4 \eta^2 H^2 \beta^H \cG^{(t)} \\
& \overset{(b)} \leq 8 d_{i,0} + 32 \eta^2 H^2 \cG^{(t)}  \overset{(c)} = 32 \eta^2 H^2 \cG^{(t)}. 
\end{aligned}
\end{equation}

For (a), we used $\beta > 1$ and $\ell < H$. For (b), we made the following observations by noting that $\eta L H \leq 1/3$:
$\beta^H = \left(1+\frac{1}{H}\right)^{H} (1+3 \eta L)^H \leq \left(1+\frac{1}{H}\right)^{2H} < 8,$
where in the last step, we used $\left(1+\frac{1}{x}\right)^x\leq e, \forall x >0$. Finally, for (c), we used $d_{i,0}=0$ since $\Delta_{i,0}^{(t)} = 0, \forall i \in [N]$. This establishes the bound claimed in the lemma. 
\end{proof}

\textbf{Proof of Theorem~\ref{thm:onestep}}.  
Plugging the result of Lemma \ref{lemma:drift} into Lemma \ref{lemm:one-step}, applying $\eta = \frac{\alpha}{H\alpha_g}$, setting $\alpha_g=1$, and simplifying using $\alpha L\leq 1$, we obtain 
\begin{equation}
\begin{aligned}
    \E\bracket{J(\bar{\theta}^{(t+1)})} &\leq \E\bracket{J(\bar{\theta}^{(t)})} - \left(\frac{\alpha}{2}-{C\alpha^2L}\right)\E\bracket{\norm{\nabla J(\Bar{\theta}^{(t)})}^2} + \bigo{\frac{\alpha^2 L\sigma^2}{NH}+\alpha^3L^2\sigma^2} + \bigo{\alpha}D^2\gamma^{2K},
\end{aligned}
\nonumber
\end{equation}
where $C \geq 1$ is some universal constant. Ensuring $4 \alpha C L \leq 1$ leads to the claim in Theorem~\ref{thm:onestep}. Note that with $\alpha_g=1$, all the requirements on $\alpha =\eta H$ above can be fulfilled by picking $\eta$ such that $\eta \leq 1/(4C LH).$ The proof follows by noting that this choice of $\eta$ suffices for Lemma~\ref{lemma:drift} to hold. 

\section{Conclusion}
We studied the problem of finding an optimal policy that performs well on average across multiple heterogeneous environments, where each environment is modeled as a Markov Decision Process (MDP). To find such an optimal policy, we formulated a federated policy optimization problem, and developed the first communication-efficient policy gradient algorithm that (i) achieves fast linear rates; (ii) provides a linear speedup in sample-complexity w.r.t. the number of agents; and (iii) incurs no heterogeneity-induced bias. As future work, we plan to study the problem of learning personalized policies in the context of multi-task/federated RL. 
\bibliographystyle{unsrt} 
\bibliography{refs_CDC}

\begin{thebibliography}{10}

\bibitem{FRL}
Jiaju Qi, Qihao Zhou, Lei Lei, and Kan Zheng.
\newblock Federated reinforcement learning: Techniques, applications, and open challenges.
\newblock {\em arXiv:2108.11887}, 2021.

\bibitem{FRL_identical_linear}
Sajad Khodadadian, Pranay Sharma, Gauri Joshi, and Siva~Theja Maguluri.
\newblock Federated reinforcement learning: Linear speedup under {M}arkovian sampling.
\newblock In {\em Int. Conf. on Machine Learning}, pages 10997--11057. PMLR, 2022.

\bibitem{dal2023federated}
Nicol{\`o} Dal~Fabbro, Aritra Mitra, and George~J Pappas.
\newblock Federated {TD} learning over finite-rate erasure channels: Linear speedup under {M}arkovian sampling.
\newblock {\em IEEE Control Systems Letters}, 7:2461--2466, 2023.

\bibitem{liu2023distributed}
Rui Liu and Alex Olshevsky.
\newblock Distributed {TD} $(0)$ with almost no communication.
\newblock {\em IEEE Control Systems Letters}, 7:2892--2897, 2023.

\bibitem{lan2023improved}
Guangchen Lan, Han Wang, James Anderson, Christopher Brinton, and Vaneet Aggarwal.
\newblock Improved communication efficiency in federated natural policy gradient via {ADMM}-based gradient updates.
\newblock {\em arXiv preprint arXiv:2310.19807}, 2023.

\bibitem{woo2023blessing}
Jiin Woo, Gauri Joshi, and Yuejie Chi.
\newblock The blessing of heterogeneity in federated {Q}-learning: Linear speedup and beyond.
\newblock In {\em International Conference on Machine Learning}, pages 37157--37216. PMLR, 2023.

\bibitem{tian2024one}
Haoxing Tian, Ioannis~Ch Paschalidis, and Alex Olshevsky.
\newblock One-shot averaging for distributed {TD} ($\lambda$) under markov sampling.
\newblock {\em IEEE Control Systems Letters}, 2024.

\bibitem{jin2022federated}
Hao Jin, Yang Peng, Wenhao Yang, Shusen Wang, and Zhihua Zhang.
\newblock Federated reinforcement learning with environment heterogeneity.
\newblock In {\em International Conf. on Artificial Intelligence and Stat.}, pages 18--37. PMLR, 2022.

\bibitem{fedtd}
Han Wang, Aritra Mitra, Hamed Hassani, George~J Pappas, and James Anderson.
\newblock Federated temporal difference learning with linear function approximation under environmental heterogeneity.
\newblock {\em arXiv:2302.02212}, 2023.

\bibitem{zhang2024finite}
Chenyu Zhang, Han Wang, Aritra Mitra, and James Anderson.
\newblock Finite-time analysis of on-policy heterogeneous federated reinforcement learning.
\newblock {\em arXiv preprint arXiv:2401.15273}, 2024.

\bibitem{sodhani}
Shagun Sodhani, Amy Zhang, and Joelle Pineau.
\newblock Multi-task reinforcement learning with context-based representations.
\newblock In {\em International Conference on Machine Learning}, pages 9767--9779. PMLR, 2021.

\bibitem{mcmahan}
Brendan McMahan, Eider Moore, Daniel Ramage, Seth Hampson, and Blaise~Aguera y~Arcas.
\newblock Communication-efficient learning of deep networks from decentralized data.
\newblock In {\em AISTATS}, pages 1273--1282. PMLR, 2017.

\bibitem{xie2023fedkl}
Zhijie Xie and Shenghui Song.
\newblock {FedKL}: Tackling data heterogeneity in federated reinforcement learning by penalizing {KL} divergence.
\newblock {\em IEEE Journal on Selected Areas in Comm.}, 41(4):1227--1242, 2023.

\bibitem{zeng2021decentralized}
Sihan Zeng, Malik~Aqeel Anwar, Thinh~T Doan, Arijit Raychowdhury, and Justin Romberg.
\newblock A decentralized policy gradient approach to multi-task reinforcement learning.
\newblock In {\em Uncertainty in Artificial Intelligence}. PMLR, 2021.

\bibitem{meiglobal}
Jincheng Mei, Chenjun Xiao, Csaba Szepesvari, and Dale Schuurmans.
\newblock On the global convergence rates of softmax policy gradient methods.
\newblock In {\em Int. Conf. on Machine Learning}, pages 6820--6829. PMLR, 2020.

\bibitem{yuan2022}
Rui Yuan, Robert~M Gower, and Alessandro Lazaric.
\newblock A general sample complexity analysis of vanilla policy gradient.
\newblock In {\em Int. Conf. on Artificial Intelligence and Statistics}, pages 3332--3380. PMLR, 2022.

\bibitem{bai2024finite}
Yitao Bai and Thinh Doan.
\newblock Finite-time complexity of incremental policy gradient methods for solving multi-task reinforcement learning.
\newblock In {\em 6th Annual Learning for Dynamics \& Control Conference}, pages 1046--1057. PMLR, 2024.

\bibitem{wang2024momentum}
Han Wang, Sihong He, Zhili Zhang, Fei Miao, and James Anderson.
\newblock Momentum for the win: Collaborative federated reinforcement learning across heterogeneous environments.
\newblock {\em arXiv preprint arXiv:2405.19499}, 2024.

\bibitem{puterman}
Martin~L Puterman.
\newblock Markov decision processes.
\newblock {\em Handbooks in Operations Research and Management Science}, 2:331--434, 1990.

\bibitem{agarwalPG}
Alekh Agarwal, Sham~M Kakade, Jason~D Lee, and Gaurav Mahajan.
\newblock On the theory of policy gradient methods: Optimality, approximation, and distribution shift.
\newblock {\em Journal of Machine Learning Research}, 22(98):1--76, 2021.

\bibitem{suttonpolicy}
Richard~S Sutton, David McAllester, Satinder Singh, and Yishay Mansour.
\newblock Policy gradient methods for reinforcement learning with function approximation.
\newblock {\em Advances in Neural Information Processing Systems}, 12, 1999.

\bibitem{mitraNIPS21}
Aritra Mitra, Rayana Jaafar, George~J Pappas, and Hamed Hassani.
\newblock Linear convergence in federated learning: Tackling client heterogeneity and sparse gradients.
\newblock {\em Advances in Neural Information Processing Systems}, 34:14606--14619, 2021.

\bibitem{scaffold}
Sai~Praneeth Karimireddy, Satyen Kale, Mehryar Mohri, Sashank Reddi, Sebastian Stich, and Ananda~Theertha Suresh.
\newblock Scaffold: Stochastic controlled averaging for federated learning.
\newblock In {\em International Conference on Machine Learning}, pages 5132--5143. PMLR, 2020.

\bibitem{fazel}
Maryam Fazel, Rong Ge, Sham Kakade, and Mehran Mesbahi.
\newblock Global convergence of policy gradient methods for the linear quadratic regulator.
\newblock In {\em Int. Conf. on Machine Learning}, pages 1467--1476. PMLR, 2018.

\bibitem{agarwal2019reinforcement}
Alekh Agarwal, Nan Jiang, Sham~M Kakade, and Wen Sun.
\newblock Reinforcement learning: Theory and algorithms.
\newblock {\em CS Dept., UW Seattle, Seattle, WA, USA, Tech. Rep}, 32, 2019.

\bibitem{xiao2022}
Lin Xiao.
\newblock On the convergence rates of policy gradient methods.
\newblock {\em Journal of Machine Learning Research}, 23(282):1--36, 2022.

\end{thebibliography}

\end{document}